\documentclass{article}

\usepackage{arxiv}

\usepackage[utf8]{inputenc}
\usepackage[T1]{fontenc}    
\usepackage{hyperref}      
\usepackage{url}
\usepackage{booktabs}       
\usepackage{amsfonts}       
\usepackage{microtype}
\usepackage{fancyhdr}  
\usepackage{graphicx}
\usepackage{latexsym}
\usepackage{booktabs}
\usepackage{enumitem}
\usepackage{color}
\usepackage{amsmath}
\usepackage{amssymb}
\usepackage{amsthm}
\usepackage{import}
\usepackage{float}
\usepackage{tabularx}
\usepackage{caption}

\newtheorem{theorem}{Theorem}
\newtheorem{lemma}[theorem]{Lemma}

\title{Tightening the Evaluation of PAC Bounds Using Formal Verification Results}

\author{
  Thomas Walker\\
  Imperial College London \\
  \texttt{thomas.walker21@imperial.ac.uk}
   \And
  Alessio Lomuscio \\
  Imperial College London \\
  \texttt{a.lomuscio@imperial.ac.uk}
}

\begin{document}
\maketitle

\begin{abstract}
Probably Approximately Correct (PAC) bounds are widely used to derive probabilistic guarantees for the generalisation of machine learning models. They highlight the components of the model which contribute to its generalisation capacity. However, current state-of-the-art results are loose in approximating the generalisation capacity of deployed machine learning models. Consequently, while PAC bounds are theoretically useful, their applicability for evaluating a model's generalisation property in a given operational design domain is limited. The underlying classical theory is supported by the idea that bounds can be tightened when the number of test points available to the user to evaluate the model increases. Yet, in the case of neural networks, the number of test points required to obtain bounds of interest is often impractical even for small problems.

In this paper, we take the novel approach of using the formal verification of neural systems to inform the evaluation of PAC bounds. Rather than using pointwise information obtained from repeated tests, we use verification results on regions around test points. We show that conditioning existing bounds on verification results leads to a tightening proportional to the underlying probability mass of the verified region.
\end{abstract}

\section{Introduction}

A key challenge in the deployment of neural networks in applications where safety is paramount is to provide assurance on their generalisation. Put succinctly, performance on the test set is not sufficient to ensure that the model will perform adequately when deployed even when the input data at run-time is in distribution with the training set. Answering these questions with sufficient confidence would facilitate deployment and certification. 

The theory of Probably Approximately Correct (PAC) bounds is a classical statistical framework used to derive probabilistic generalisation bounds for learning algorithms. The generalisation capacity of a learning algorithm, such as a neural network, is the ability to extrapolate performance from a random training sample to the underlying distribution of the data. The generalisation gap quantifies generalisation capacity as the difference between the true error of the learning algorithm and the error of the learning algorithm on an independent evaluating sample. PAC bounds utilise the structure of the learning algorithm along with the assumptions on the evaluating sample to derive probabilistic bounds on the generalisation gap. Generally speaking, a PAC bound states that for a particular configuration of the learning algorithm, we have $$\text{true error}\leq\text{eval error}+\frac{\text{complexity}}{m}$$with high probability, where $m$ is the number of samples in the evaluating sample. In other words, $$\text{generalisation gap}=\text{true error}-\text{eval error}\leq\frac{\text{complexity}}{m}.$$The complexity term identifies how certain properties of the learning algorithm influence its generalisation capacity. The fidelity of the complexity terms is essential for tight generalisation bounds at scale, as increasing the sample size of the evaluating sample yields diminishing returns. In the classical setting, the complexity term may involve quantities such as the Vapnik-Chervonenkis dimension \cite{bartlett_nearly-tight_2017}, however, for neural networks more appropriate complexity terms have been derived \cite{arora_stronger_2018}. 

It is standard to consider the neural network as a random sample from a distribution of networks and derive generalisation bounds in a Bayesian setting. The PAC-Bayes framework \cite{mcallester_pac-bayesian_1998} combines PAC learning theory and Bayesian reasoning. In this setting, bounds have been successfully applied in practice by optimizing its components, such as the prior distribution. In \cite{dziugaite_computing_2017} the prior distribution is optimised by understanding how stochastic gradient descent trains a neural network, resulting in the initial non-vacuous PAC-Bayes bounds for neural networks. Subsequently, \cite{zhou_non-vacuous_2019} successfully contextualised the work of \cite{arora_stronger_2018} to obtain non-vacuous bounds at the Image-Net level. An open challenge in this area concerns the derivation of tight generalisation bounds for neural networks that can scale to large models.

So far, generalisation bounds have been optimised under the PAC-Bayes framework by utilising information about neural networks at individual samples. The convergence to tight generalisation bounds under this process is extremely slow, and, in practice, still leads to bounds that are too loose to provide a significant prediction on model behaviour. Therefore, rather than using individual training samples, we suggest testing the neural network in regions of the input space to obtain information about the neural network that is likely to generate tighter generalisation bounds.

In this paper, we explore a novel strategy for evaluating generalisation bounds. We use information obtained through neural network verification to derive provably tighter generalisation bounds.

In the rest of this section, we summarise related work. In Section \ref{sec:preliminaries} we introduce the theory of PAC bounds and the basics of neural network verification. In Section \ref{sec:verification_incorporated_pac_bounds} we formulate the problem using an assumption which we then utilise to obtain provably tighter generalisation bounds. We then understand how formal verification can satisfy our assumption to yield tighter generalisation bounds. In Section \ref{sec:conclusion} we summarise our results and identify important areas for future work.

\subsection{Related Work}

It has been shown theoretically in \cite{xu_robustness_2010} that a learning algorithm generalises if and only if it is robust against perturbations. While this work theoretically supports the idea of tightening generalisation bounds using verification, it does not explore this connection in a way that is useful in practice given the current methods of formal verification. It formulates robustness as a condition over the entire input space, whereas methods of neural network verification certify robustness in neighbourhoods of points in the input space \cite{weng_proven_2019}. In our work, we provide an alternative approach to theoretically incorporate robustness into generalisation bounds such that the extent to which robustness improves generalisation bounds can be empirically investigated. 

In our work we look at how robustness can inform generalisation bounds, however, there is a closely related body of work that uses the PAC-Bayes framework to train robust neural networks. On the one hand, \cite{viallard_pac-bayes_2021} uses the PAC-Bayes framework to obtain a probabilistic bound on the occurrence of adversarial samples. Adversarial samples are slight perturbations of samples in the input space that lead to qualitatively different behaviour in the output of the neural network. Similarly, work looks at training neural networks to be robust against adversarial samples \cite{madry_towards_2019,zhang_theoretically_2019}. However, with such methods, there is an observed comprise with the generalisation capacity of the neural network \cite{rice_overfitting_2020}. Therefore, \cite{jin_enhancing_2022,wang_improving_2022} use the PAC-Bayes framework to inform the methods of adversarial robustness training to help maintain the generalisation capacity of the trained neural networks. This line of work uses the idea that a neural network that generalises should be robust, to train robust neural networks. In this paper, we use the idea that a robust neural network should generalise, to strengthen generalisation bounds.

As discussed earlier, there is work that aims at tightening generalisation bounds for neural networks by optimising different components of the bound rather than relying on scaling up the size of the evaluating set. \cite{dziugaite_computing_2017} uses artefacts from stochastic gradient descent to optimise the prior distribution of PAC-Bayes bounds, whereas \cite{zhou_non-vacuous_2019} understands how weight perturbations propagate through the neural network to contextualise the work of \cite{arora_stronger_2018}. However, none of these approaches use verification artefacts, as we do here, to obtain tighter bounds.

There is work in the literature in developing methods of formal verification of neural networks \cite{katz_reluplex_2017,botoeva_efficient_2020,bunel_lagrangian_2020,ferrari_complete_2022}, which have been applied to certify the performance of neural networks in applications \cite{kouvaros_formal_2021}. However, none of them explore how the results of their methods could be incorporated into generalisation bounds. Here we intend to provide a framework to leverage the success of this body of work to improve the tightness of evaluated generalisation bounds.

To our knowledge, this is the first approach to bring verification into the evaluation of generalisation bounds.

\section{Preliminaries}\label{sec:preliminaries}

\subsection{PAC Bounds}

Let $\mathcal{Z}=\mathcal{X}\times\mathcal{Y}$ be the input space where $\mathcal{X}$ is the feature space and $\mathcal{Y}$ is the label space. We assume an unknown distribution $\mathcal{D}$ is supported on the input space. A neural network $h:\mathcal{X}\to\mathcal{Y}$ is trained to encode the distribution $\mathcal{D}$ on $\mathcal{Z}$. A neural network is represented by $h_{\mathbf{w}}\in\mathcal{H}$ where $\mathbf{w}\in\mathcal{W}$ is a weight vector. The sets $\mathcal{W}$ and $\mathcal{H}$ will be referred to as the parameter space and the hypothesis set respectively. A loss function $l:\mathcal{Y}\times\mathcal{Y}\to[0,C]$ is used to assess the quality of a trained neural network. For $z=(x,y)\in\mathcal{Z}$ we identify the quantity $l_z(\mathbf{w}):=l(h_{\mathbf{w}}(x),y)$ as the error of the input under the neural network $h_{\mathbf{w}}$. The true error of the neural network on the input space is then $$R(\mathbf{w})=\mathbb{E}_{z\sim\mathcal{D}}(l_z(\mathbf{w})),$$ which is dependent on the unknown distribution $\mathcal{D}$ and hence unknown. Thus, in practice, we work with an evaluated error obtained from a sample $S=\{z_i\}_{i=1}^m=\{(x_i,y_i)\}_{i=1}^m$, which is assumed to consist of $m$ independent and identically distributed samples from $\mathcal{D}$. Specifically, the evaluated error is $$\hat{R}(\mathbf{w})=\frac{1}{m}\sum_{i=1}^ml_{z_i},$$which under the assumptions of the evaluating sample satisfies $$\mathbb{E}_{S\sim\mathcal{D}^m}\left(\hat{R}(\mathbf{w})\right)=R(\mathbf{w}).$$

Under these conditions, we have the following.
\begin{theorem}[\cite{alquier_user-friendly_2023}]\label{thm:pac_bound}
    Let $\vert\mathcal{W}\vert=M<\infty$, $\delta\in(0,1)$ and $\mathbf{w}\in\mathcal{W}$. Then$$\mathbb{P}_{S\sim\mathcal{D}^m}\left(R(\mathbf{w})\leq\hat{R}(\mathbf{w})+C\sqrt{\frac{\log\left(\frac{M}{\delta}\right)}{2m}}\right)\geq1-\delta.$$
\end{theorem}

Theorem \ref{thm:pac_bound} states that for a trained model, parameterised by $\mathbf{w}\in\mathcal{W}$, the generalisation gap on an evaluating sample $S$ of size $m$ independently sampled from $\mathcal{D}$ is bounded by $C\sqrt{\frac{\log\left(\frac{M}{\delta}\right)}{2m}}$ with probability, or confidence, $1-\delta$. Intuitively the generalisation bound should be decreasing in $m$ and $\delta$, since for larger $m$ the evaluated error should be closer to the true error, and for larger $\delta$ the bound can hold with less confidence. We also note the factor of $M$ in the generalisation bound of Theorem \ref{thm:pac_bound}, which arises due to a union-bound argument that allows the result to hold uniformly over all possible configurations of the learning algorithm. As the space of weight vectors for a neural network is large, such uniform results in this setting are vacuous. Therefore, state-of-the-art PAC bounds are derived under the PAC-Bayes framework to eliminate this factor of $M$ by sampling $\mathbf{w}$ from a posterior distribution supported on $\mathcal{W}$. In this paper we will suppose that the configuration of the neural network is determined independently of the sample $S$; this will eliminate the factor of $M$ in the generalisation bound of Theorem \ref{thm:pac_bound}. 

The above assumption allows us to apply Theorem \ref{thm:pac_bound} non-vacuously in the setting of neural networks under the condition that the sample we use to determine the generalisation capacity of the neural network is different from the sample used to train the neural network. We adopt this approach for simplicity, with the understanding that one can readily extend the explorations of this paper to the more general application of PAC bounds to neural networks.

\subsection{Verification}

We will consider the formal verification of a neural network trained to perform $k$-class classification in $d$-dimensional feature space. A neural network $h_{\mathbf{w}}:\mathbb{R}^d\to\mathbb{R}^k$ correctly classifies $(x,y)\in\mathcal{Z}$ if $\mathrm{argmax}(h_{\mathbf{w}}(x))=y$. Assume a neural network $h_{\mathbf{w}}$ correctly classifies $(x,y)\in\mathcal{Z}$, then the neural network verification problem is determining whether for a given neighbourhood $\mathcal{E}\subseteq\mathcal{X}$ of $x$ we have $h_{\mathbf{w}}\left(\tilde{x}\right)=y$ for all $\tilde{x}\in\mathcal{E}$. Often one takes the neighbourhoods $\mathcal{E}$ to be $\epsilon$-balls around test points under some metric, $B_{\epsilon}(x)$. By verifying for progressively larger values of $\epsilon$, one can determine the maximum value of $\epsilon$ for which $h_{\mathbf{w}}\left(\tilde{x}\right)=y$ for every $\tilde{x}\in B_{\epsilon}(x)$.

\section{Conditioning PAC Bounds on Formal Verification Results}\label{sec:verification_incorporated_pac_bounds}

In this section, we develop the theoretical underpinnings to improve both the quantitative value of the PAC bounds as well as the confidence associated with them by conditioning them on results obtained through formal verification. In the next subsection, we provide an intuition for the approach and formalise the problem. Next, we derive Theorem~\ref{thm:tightened_bound} which captures the improvement in the bounds obtained through verification. We then prove Theorem~\ref{thm:updated_confidence} which quantifies the increase in confidence in the bounds obtained. We then evaluate the obtained results.

\subsection{Problem Formalisation}

In the over-parameterised setting of neural networks, it is often possible to identify a $\mathbf{w}\in\mathcal{W}$ such that $\hat{R}(\mathbf{w})$ is near zero on any evaluating sample $S$ \cite{zhang_understanding_2017}. Assuming this to be the case, Theorem \ref{thm:pac_bound} gives a bound on the true error of the neural network that holds with probability, or confidence, $1-\delta$. Moreover, Theorem \ref{thm:pac_bound} indicates how the bound changes as the size of the evaluating sample increases, in particular, the bound decreases on the order of $\frac{1}{\sqrt{m}}$. Thus continually adding data points to the evaluating sample yields diminishing returns in terms of improving the tightness of the generalisation bound. We can also quantify the benefits of adding data points by observing the improved confidence for a fixed bound.

The diminishing return obtained by simply scaling the evaluating sample has motivated more effective use of sampled data to generate tight generalisation bounds \cite{dziugaite_computing_2017,zhou_non-vacuous_2019}.

In the following, we develop an alternative approach, based on formal verification, which, as we will see, improves the tightness of the bounds.

Suppose for a neural network $h_{\mathbf{w}}$ we can identify a region $\Delta\subseteq\mathcal{Z}$ for which the error of our model is zero, $l_{\Delta}(\mathbf{w})=0$. For example, with the $0$-$1$ loss function applied to a classification neural network, this assumption states that the neural network correctly classifies all the data points in the region $\Delta$. It is important to note that $\Delta$ is taken to be independent of the evaluating set $S$. This assumption provides information regarding the shape of the distribution $\mathcal{D}$, something that individual data points cannot provide.  However, we are not able to explicitly calculate 
$$p_{\Delta}=\mathbb{P}_{z\sim\mathcal{D}}(z\in\Delta)=\int_{z\in\Delta}\mathcal{D}(z)\,\mathrm{d}z,$$the probability mass of the region $\Delta$. This is problematic in a probabilistic setting as $p_{\Delta}$ quantifies the significance of the region $\Delta$ under the distribution $\mathcal{D}$. However, in the following, we will assume that $p_{\Delta}$ is known. In practice we expect one to be able to approximate $p_{\Delta}$ using a confidence interval. We leave the details of this practicality for future work, however, in Section \ref{sec:approximating_p_Delta} we show how one can theoretically incorporate this approximation into our proposed framework.

\subsection{Improving Bound Tightness}
We now proceed to condition PAC bounds on conclusions of formal verification as described above.  More specifically, we can use the law of total probability on the event that $k$ of the $m$ samples are in the region $\Delta$, denoted $\Delta_k$. Conditioned on the event $\Delta_{m-k}$, the evaluated error should be the average of the loss of only $k$ samples, as the loss of $m-k$ samples is zero. Therefore, $\hat{R}(\mathbf{w})$ is no longer the average of the loss, instead $\frac{m}{k}\hat{R}(\mathbf{w})$ is the average of the loss. Under the independent and identically distributed assumption of our sample we know that $$\mathbb{P}_{S\sim\mathcal{D}^m}(\Delta_{m-k})=\binom{m}{k}p_{\Delta}^{m-k}\left(1-p_{\Delta}\right)^k,$$hence $$\mathbb{P}_{S\sim\mathcal{D}^m}\left(R(\mathbf{w})>\hat{R}(\mathbf{w})+s\vert l_{\Delta}(\mathbf{w})=0\right)\\=\sum_{k=1}^m\binom{m}{k}\mathbb{P}_{S\sim\mathcal{D}^m}\left(R(\mathbf{w})>\frac{m}{k}\hat{R}(\mathbf{w})+s\right)p_{\Delta}^{m-k}\left(1-p_{\Delta}\right)^k.$$Rearranging the expression in the probability of the right hand allows us to apply Theorem \ref{thm:pac_bound} to deduce that 

\begin{equation}\label{eq:implicit_bound}
    \mathbb{P}_{S\sim\mathcal{D}^m}\left(R(\mathbf{w})>\hat{R}(\mathbf{w})+s\vert l_{\Delta}(\mathbf{w})=0\right)\leq\sum_{k=1}^m\binom{m}{k}\exp\left(-2m\left(\frac{m-k}{Ck}\hat{R}(\mathbf{w})+s\right)^2\right)p_{\Delta}^{m-k}\left(1-p_{\Delta}\right)^k.
\end{equation}

Note how we can drop the factor of $M$ in the right-hand side of Theorem \ref{thm:pac_bound} as we are operating under the assumption that $\mathbf{w}$ is fixed and independent of $S$. One can now set the right-hand side of \eqref{eq:implicit_bound} equal to $\delta$ and solve for $s$ to deduce that $$\mathbb{P}_{S\sim\mathcal{D}^m}\left(R(\mathbf{w})\leq\hat{R}(\mathbf{w})+s\vert l_{\Delta}(\mathbf{w})=0\right)\geq1-\delta.$$Letting $p_{\Delta}=0$ in the right-hand side of \eqref{eq:implicit_bound} we recover Theorem \ref{thm:pac_bound}. Similarly, letting $p_{\Delta}=1$ we get that $\delta=0$ as expected as $R(\mathbf{w})=\hat{R}(\mathbf{w})$, so the inequality $R(\mathbf{w})\leq\hat{R}(\mathbf{w})+s$ holds trivially. However, one cannot solve the right-hand side of \eqref{eq:implicit_bound} to determine an explicit form for $s$. Instead one has to use concentration inequalities to derive a closed-form solution for $s$. 

\begin{theorem}\label{thm:tightened_bound}
    Let $\mathbf{w}\in\mathcal{W}$ and $\delta\in(0,1)$. Then $$\mathbb{P}_{S\sim\mathcal{D}^m}\left(R(\mathbf{w})\leq\hat{R}(\mathbf{w})+CB(m,p_{\Delta},\delta)\Big\vert l_{\Delta}(\mathbf{w})=0\right)\geq1-\delta$$ for $$B(m,p_{\Delta},\delta)=\sqrt{\frac{\log\left(\frac{(1-p_{\Delta})+\sqrt{(1-p_{\Delta})^2+4\delta^{\frac{1}{m}}p_{\Delta}}}{2\delta^\frac{1}{m}}\right)}{2}}.$$
\end{theorem}

\begin{proof}
    Using the law of total expectation on the events $\Delta_k$ for $k=0,\dots,m$ and Lemma \ref{lem:exponential_bound_exponential_xpectation}\footnote{The statement and proof of Lemma \ref{lem:exponential_bound_exponential_xpectation} can be found in Section \ref{sec:supporting_lemmas}.}, we deduce that
    \begin{align*}
        \mathbb{E}_{S\sim\mathcal{D}^m}&\left(\exp\left(t\sum_{i=1}^m(\mathbb{E}(l_{z_i}(\mathbf{w}))-l_{z_i}(\mathbf{w}))\right)\bigg\vert l_{\Delta}(\mathbf{w})=0\right)\\&\leq\sum_{k=0}^m\binom{m}{k}\exp\left(\frac{kt^2C^2}{8}\right)p_{\Delta}^{m-k}(1-p_{\Delta})^k,
    \end{align*}
    where we have noted that when conditioned on the event $\Delta_{m-k}$ the index of the sum in the exponential is reduced to $i=1,\dots,k$. Applying Markov's inequality and using the binomial theorem we get that
    \begin{align}\label{eq:probability_bound}
        &\mathbb{P}_{S\sim\mathcal{D}^m}\left(R(\mathbf{w})>\hat{R}(\mathbf{w})+s\big\vert l_{\Delta}(\mathbf{w})=0\right)\nonumber\\&\leq\frac{1}{\exp(mts)}\sum_{k=0}^m\binom{m}{k}\exp\left(\frac{kt^2C^2}{8}\right)p_{\Delta}^{m-k}(1-p_{\Delta})^k\nonumber\\&=\frac{1}{\exp(mts)}\left(p_{\Delta}+(1-p_{\Delta})\exp\left(\frac{t^2C^2}{8}\right)\right)^m
    \end{align}
    for $t\geq0$. In particular, for $t=\frac{4s}{C^2}$ we get
    \begin{align*}
        &\mathbb{P}_{S\sim\mathcal{D}^m}\left(R(\mathbf{w})>\hat{R}(\mathbf{w})+s\Big\vert l_{\Delta}(\mathbf{w})=0\right)\\&\leq\frac{1}{\exp\left(\frac{4ms^2}{C^2}\right)}\left(p_{\Delta}+(1-p_{\Delta})\exp\left(\frac{2s^2}{C^2}\right)\right)^m.
    \end{align*} 
    Therefore, to achieve a confidence level of $1-\delta$ we let $$\delta=\frac{1}{\exp\left(\frac{4ms^2}{C^2}\right)}\left(p_{\Delta}+(1-p_{\Delta})\exp\left(\frac{2s^2}{C^2}\right)\right)^m.$$ Which we can then rearrange to determine that $$s=\sqrt{\frac{C^2\log\left(\frac{(1-p_{\Delta})+\sqrt{(1-p_{\Delta})^2+4\delta^{\frac{1}{m}}p_{\Delta}}}{2\delta^\frac{1}{m}}\right)}{2}}$$provides our $1-\delta$ confident bound on the true error. 
\end{proof}

Theorem \ref{thm:tightened_bound} states that for a trained neural network $h_{\mathbf{w}}$ that achieves zero loss in a region $\Delta\subseteq\mathcal{Z}$, the true error of the neural network is bounded by the sum of the empirical error of the neural network on an independent and identically distributed sample of size $m$ and $CB(m,p_{\Delta},\delta)$ with probability $1-\delta$.

Note that letting $p_{\Delta}=0$ in Theorem \ref{thm:tightened_bound} we recover the statement of  Theorem \ref{thm:pac_bound}. With $p_{\Delta}=1$ we see that $B(m,p_{\Delta},\delta)>0$ which is not optimal as noted previously. This sub-optimal bound is the trade-off we incur when using concentration inequalities, such as Lemma \ref{lem:exponential_bound_exponential_xpectation}, to derive an explicit bound.

\subsection{Improving Bound Confidence}
We can also use the intuition above to improve the confidence associated with the bounds. This will provide a standard method to compare the relative tightening of different bounds after conditioning with the assumption $l_{\Delta}(\mathbf{w})=0$.

\begin{theorem}\label{thm:updated_confidence}
    Let $\mathbf{w}\in\mathcal{W}$ and $\delta\in(0,1)$. Then$$\mathbb{P}_{S\sim\mathcal{D}^m}\left(R(\mathbf{w})\leq\hat{R}(\mathbf{w})+C\sqrt{\frac{\log\left(\frac{1}{\delta}\right)}{2m}}\bigg\vert l_{\Delta}(\mathbf{w})=0\right)\geq1-\left(\sum_{k=1}^m\binom{m}{k}\delta_kp_{\Delta}^{m-k}(1-p_{\Delta})^k\right)$$where $$\delta_k=\exp\left(-2m\left(\frac{m-k}{Ck}\hat{R}(\mathbf{w})+\sqrt{\frac{\log\left(\frac{1}{\delta}\right)}{2m}}\right)^2\right)$$for $k=1,\dots,m$.
\end{theorem}

\begin{proof}
    Using the law of total probability on the events $\Delta_k$ for $k=0,\dots,m$ we deduce that 
    \begin{align*}
        \mathbb{P}_{S\sim\mathcal{D}^m}\Bigg(R(\mathbf{w})>\hat{R}(\mathbf{w})+&C\sqrt{\frac{\log\left(\frac{1}{\delta}\right)}{2m}}\Bigg\vert l_{\Delta}(\mathbf{w})=0\Bigg)\\&=\sum_{k=1}^m\binom{m}{k}\mathbb{P}_{S\sim\mathcal{D}^m}\left(R(\mathbf{w})>\frac{m}{k}\hat{R}(\mathbf{w})+C\sqrt{\frac{\log\left(\frac{1}{\delta}\right)}{2m}}\right)\cdot p_{\Delta}^{m-k}(1-p_{\Delta})^k,
    \end{align*}
    where we have used that the evaluated error conditioned on the event $\Delta_{m-k}$ is $\frac{m}{k}\hat{R}(\mathbf{w})$. Now we rearrange the term in the probability of the right-hand side so that we can apply Theorem \ref{thm:pac_bound}. More specifically, we note that
    \begin{align*}
        \mathbb{P}_{S\sim\mathcal{D}^m}\Bigg(R(\mathbf{w})>\frac{m}{k}\hat{R}(\mathbf{w})+&C\sqrt{\frac{\log\left(\frac{1}{\delta}\right)}{2m}}\Bigg)\\&=\mathbb{P}_{S\sim\mathcal{D}^m}\left(R(\mathbf{w})>\hat{R}(\mathbf{w})+\frac{m-k}{k}\hat{R}(\mathbf{w})+C\sqrt{\frac{\log\left(\frac{1}{\delta}\right)}{2m}}\right),
    \end{align*}
    so that by setting $$\delta_k=\exp\left(-2m\left(\frac{m-k}{Ck}\hat{R}(\mathbf{w})+\sqrt{\frac{\log\left(\frac{1}{\delta}\right)}{2m}}\right)^2\right)$$ for $k=1,\dots, m$, we deduce using Theorem \ref{thm:pac_bound} that $$\mathbb{P}_{S\sim\mathcal{D}^m}\left(R(\mathbf{w})>\hat{R}(\mathbf{w})+C\sqrt{\frac{\log\left(\frac{1}{\delta}\right)}{2m}}\Bigg\vert l_{\Delta}(\mathbf{w})=0\right)\leq\sum_{k=1}^m\binom{m}{k}\delta_kp_{\Delta}^{m-k}(1-p_{\Delta})^k.$$
\end{proof}

Theorem \ref{thm:updated_confidence} demonstrates how for a fixed bound, the confidence of the bound holds changes with the conclusions of verification. Indeed, the changes are an improvement as $\delta_k\leq\delta$. More specifically, we observe that the conclusions of verification are more effective in improving confidence if the evaluated error of the neural network is higher. Moreover, with $p_{\Delta}=0$ in Theorem \ref{thm:updated_confidence} we recover the statement of Theorem \ref{thm:pac_bound}. With $p_{\Delta}=1$ we get full confidence in our bound as desired. We note the resemblance of the updated confidence in Theorem \ref{thm:updated_confidence} and the implicit bound \eqref{eq:implicit_bound}, evidence of the fact that these conditioning processes are equivalent.

\subsection{Evaluation of Results}\label{sec:method_evaluation}

Recall that when tightening the bounds, we saw that we could not derive closed forms for the updated bounds. We can either work with the implicit forms through numerical methods or utilise concentration inequalities to obtain sub-optimal closed-form bounds. Moreover, we note that the empirical risk $\hat{R}(\mathbf{w})$ appears in the updated bound and confidence which is to be expected, for if the empirical risk is large then the added information that the loss is zero in a region should be more significant than if the empirical risk was already low. It follows that the practical usefulness of our results depends on our capacity to verify probabilistically significant regions of the input space.

\begin{table}[H]
    \centering
    \captionsetup{width=0.9\columnwidth}
      \caption{The values of $p_{\Delta}$ necessary such that the sum in the right-hand side of the inequality in Theorem \ref{thm:updated_confidence} is a certain percentage decrease of $\delta=0.05$. These values are determined for when $m=100$, $\delta=0.05$ and $\hat{R}(\mathbf{w})=0.05$.}
      \label{tab:p_Delta_for_confidence_improvement}
      \begin{tabular}{|l|l|}
        \hline
        Decrease in $\delta$ & $p_{\Delta}$ \\
        \hline
        $10\%$ & $0.045$ \\
        $20\%$ & $0.0823$ \\
        $30\%$ & $0.1244$ \\
        $50\%$ & $0.2127$ \\
        $75\%$ & $0.3423$ \\
        \hline
      \end{tabular}
\end{table}

\begin{table}[h]
    \centering
    \captionsetup{width=0.9\columnwidth}
      \caption{The values of $p_{\Delta}$ required such that the bound in \eqref{eq:implicit_bound} is a certain percentage decrease of the original bound. These values are determined for when $m=100$, $\delta=0.05$ and $\hat{R}(\mathbf{w})=0.05$.}
      \begin{tabular}{|l|l|}
        \hline
        Decrease in bound & $p_{\Delta}$ \\
        \hline
        $10\%$ & $0.045$ \\
        $20\%$ & $0.0823$ \\
        $30\%$ & $0.1244$ \\
        $50\%$ & $0.2127$ \\
        $75\%$ & $0.3423$ \\
        \hline
      \end{tabular}
      \label{tab:p_Delta_for_bound_improvement}
\end{table}

From Tables \ref{tab:p_Delta_for_confidence_improvement} and \ref{tab:p_Delta_for_bound_improvement} we verify that improving the tightness of the bound and improving the confidence of the bound are equivalent processes. Generally, we see that to obtain significant improvements one ought to identify a region of the input space with a probability mass of around $0.1$ under the distribution $\mathcal{D}$. 

\begin{table}[h]
    \centering
    \captionsetup{width=0.9\columnwidth}
      \caption{The percentage increase in the size of $m$, from $m=1000$, in Theorem \ref{thm:pac_bound} to decrease the value of the bound by the corresponding percentage.}
      \begin{tabular}{|l|l|}
        \hline
        Decrease in bound & Increase in test size \\
        \hline
        $10\%$ & $23\%$ \\
        $20\%$ & $56\%$ \\
        $30\%$ & $104\%$ \\
        $50\%$ & $4300\%$ \\
        $75\%$ & $15000\%$ \\
        \hline
      \end{tabular}
      \label{tab:test_size_increase}
\end{table}
 
From Table \ref{tab:test_size_increase} we see that to observe the corresponding improvements in the bound of Theorem \ref{thm:pac_bound}, one would have to significantly increase the size of the sample on which they evaluate the neural network.

\section{Conclusion}\label{sec:conclusion}

We have shown theoretically that probabilistic generalisation bounds can be tightened by understanding the robustness of the neural network. In particular, we have shown that the extent to which formal verification improves the tightness of the generalisation bound is dependent on the probability mass of the verified region under $\mathcal{D}$, and the evaluated error of the neural network on a random sample drawn from $\mathcal{D}$.

Our work provides a novel approach to tightening the evaluation of generalisation bounds using data, as we depart from traditional pointwise approaches. We expect that this framework will help mitigate the data-intensive drawback of current state-of-the-art generalisation bounds, facilitating the application of generalisation bounds to certify the deployment of systems in data-limited and safety-critical domains.

\subsection{Future Work}

With our framework, we theoretically incorporated neural network verification into the evaluation of generalisation. We leave it for future work to understand how this incorporation can be realised in practice.

There is a body of work developing methods to produce probabilistic verification certificates in regions of the input space \cite{weng_proven_2019}. It would be useful to explore how these results could be theoretically incorporated into our framework.

\bibliographystyle{unsrt}  
\bibliography{mybibfile}  

\section{Appendix}

\subsection{Approximating $p_{\Delta}$}\label{sec:approximating_p_Delta}

Throughout the paper, we used $p_{\Delta}$ as a theoretical device to explore its interaction with generalisation bounds. In practice, getting an exact value for $p_{\Delta}$ is not feasible since $\mathcal{D}$ is unknown. However, we show that an approximation of $p_{\Delta}$ can be derived from the conclusions of neural network verification. Suppose that $\Delta$ is a verified region of the input space provided by neural network verification and that $S_A$ is an independent and identically distributed sample from $\mathcal{D}$ of size $m_A$. For $z_i=(x_i,y_i)\in S_A$ let $$Z_i:=\begin{cases}1&z_i\in\Delta\\0&z_i\in\mathcal{Z}\setminus\Delta,\end{cases}$$which is a Bernoulli random variable with parameter $p_{\Delta}$. An estimator of $p_{\Delta}$ is $$\hat{p}_{\Delta}=\frac{1}{m_A}\sum_{i=1}^{m_A}Z_i.$$We can then derive a confidence interval for $p_{\Delta}$. The $1-\alpha$ one-sided Clopper-Pearson confidence interval for $\hat{p}_{\Delta}$ is $$\left[q_B\left(\alpha,m_A\hat{p}_{\Delta},m_A-m_A\hat{p}_{\Delta}+1\right),1\right],$$where $q_B(\cdot,\beta_1,\beta_2)$ is the quantile function for the beta distribution with shape parameters $\beta_1$ and $\beta_2$ \cite{helwig_inference_2020}. Assuming that our generalisation bounds are decreasing in $p_{\Delta}$ and $\delta$, the loosest bound is realised by taking the lower end of this interval, $p_L:=q_B\left(\alpha,m_A\hat{p}_{\Delta},m_A-m_A\hat{p}_{\Delta}+1\right)$. Hence, 
\begin{align*}
    &\mathbb{P}_{S\sim\mathcal{D}^m}\left(R(\mathbf{w})>\hat{R}(\mathbf{w})+B(p_L)\right)\\&=\mathbb{P}_{S\sim\mathcal{D}^m}\left(R(\mathbf{w})>\hat{R}(\mathbf{w})+B(p_L)\big\vert p_{\Delta}\geq p_L\right)\mathbb{P}_{S\sim\mathcal{D}^m}(p_{\Delta}\geq p_L)\\&\quad+\mathbb{P}_{S\sim\mathcal{D}^m}\left(R(\mathbf{w})>\hat{R}(\mathbf{w})+B(p_L)\big\vert p_{\Delta}<p_L\right)\mathbb{P}_{S\sim\mathcal{D}^m}(p_{\Delta}<p_L)\\&\leq\delta+\alpha(1-\delta).
\end{align*}
The intention is to choose $\alpha$ sufficiently small so that the added $\alpha(1-\delta)$ term is small. Decreasing $\alpha$ will decrease $p_L$ and hence increase $B(p_L)$. So to maintain a tight bound we must supplement the decreasing of $\alpha$ with an increase in $m_A$, as this will increase $p_L$ and subsequently decrease $B(p_L)$. In the machine learning paradigm, this may involve holding out some data from the training set to approximate $p_{\Delta}$.

\subsection{Supporting Lemmas}\label{sec:supporting_lemmas}

\begin{lemma}[\cite{scott_hoeffdings_2014}]\label{lem:exponential_bound_exponential_xpectation}
    Let $U_1,\dots,U_n$ be independent random variables taking values in an interval $[a,b]$. Then for any $t>0$ we have that $$\mathbb{E}\left(\exp\left(t\sum_{i=1}^n\left(U_i-\mathbb{E}(U_i)\right)\right)\right)\leq\exp\left(\frac{nt^2(b-a)^2}{8}\right).$$
\end{lemma}

\begin{proof}
    For $s>0$ the function $x\mapsto e^{sx}$ is convex  so that $$e^{sx}\leq\frac{x-a}{b-a}e^{sb}+\frac{b-x}{b-a}e^{sa}.$$ Let $V_i=U_i-\mathbb{E}(U_i)$, then as $\mathbb{E}(V_i)=0$ it follows that $$\mathbb{E}\left(\exp(sV_i)\right)\leq\frac{b}{b-a}e^{sa}-\frac{a}{b-a}e^{sb}.$$ With $p=\frac{b}{b-a}$ and $u=(b-a)s$ consider $$\psi(u)=\log\left(pe^{sa}+(1-p)e^{sb}\right)=(p-1)u+\log\left(p+(1-p)e^u\right).$$ This is a smooth function so that by Taylor's theorem we have that for any $u\in\mathbb{R}$ there exists $\xi=\xi(u)\in\mathbb{R}$ such that $$\psi(u)=\psi(0)+\psi^\prime(0)u+\frac{1}{2}\psi^{\prime\prime}(\xi)u^2.$$ As $$\psi^\prime(u)=(p-1)+1-\frac{p}{p+(1-p)e^u}$$ we have that $\psi(0)=0$ and $\psi^\prime(0)=0$. Furthermore, as $$\psi^{\prime\prime}(u)=\frac{p(1-p)e^u}{(p+(1-p)e^u)^2}$$and$$\psi^{(3)}(u)=\frac{p(1-p)e^u(p+(1-p)e^u)(p-(1-p)e^u)}{(p+(1-p)e^u)^2}$$ we see that $\psi^{\prime\prime}(u)$ has a stationary point at $u^*=\log\left(\frac{p}{p-1}\right)$. For $u$ slightly less than $u^*$ we have $\psi^{(3)}(u)>0$ and for $u$ slightly larger than $u^*$ we have $\psi^{(3)}(u)<0$. Therefore, $u^*$ is a maximum point and so $$\psi^{\prime\prime}(u)\leq\psi^{\prime\prime}(u^*)=\frac{1}{4}.$$ Hence, $\psi(u)\leq\frac{u^2}{8}$ which implies that $$\log\left(\mathbb{E}\left(\exp(sV_i)\right)\right)\leq\frac{u^2}{8}=\frac{s^2(b-a)^2}{8}.$$ Therefore, 
    \begin{align*}
        \mathbb{E}\left(\exp\left(t\sum_{i=1}^n\left(U_i-\mathbb{E}(U_i)\right)\right)\right)&=\prod_{i=1}^n\mathbb{E}\left(\exp\left(t(U_i-\mathbb{E}(U_i))\right)\right)\\&\leq\prod_{i=1}^n\exp\left(\frac{t^2(b-a)^2}{8}\right)\\&\leq\exp\left(\frac{nt^2(b-a)^2}{8}\right).
    \end{align*} 
\end{proof}

\end{document}